\documentclass[11pt]{article}

\usepackage{amsmath}
\usepackage{graphicx}
\usepackage{epstopdf}
\usepackage{dsfont}
\usepackage{amsthm}

\newcommand{\GM}[1]{\text{\textnormal{GM}}^{*}{#1}}

\newcommand{\CL}[1]{\text{\textnormal{CL}}^{*}{#1}}

\newcommand{\cost}[4]{\textnormal{c}(#1^{#2}, #3^{#4})}
\newcommand{\metriccost}[4]{\textnormal{c}_{M}(#1^{#2}, #3^{#4})}

\newcommand{\argmin}{\operatornamewithlimits{argmin}}
\newcommand{\sen}[1]{\| #1 \|^{2}}
\newcommand{\argp}[2]{\mathbf{#1_{#2}}\mathbf{V_{#2}}}
\newcommand{\scalprod}[2]{\langle #1, #2\rangle}
\newcommand{\email}[1]{\emph{email:} \texttt{#1}}

\newtheorem{theorem}{Theorem}

\newtheorem{corollary}{Corollary}

\title{On the Relation Between the Common Labelling and the Median Graph\thanks{This work was published in STRUCTURAL, SYNTACTIC, AND STATISTICAL PATTERN RECOGNITION
Lecture Notes in Computer Science, 2012, Volume 7626/2012, 107-115, DOI: 10.1007/978-3-642-34166-31\_2. The original publication is available at http://www.springerlink.com/content/e524g4483g146383/}}

\author{Nicola Rebagliati\thanks{VTT Technical Research Center of Finland.
   \email{nicola.rebagliati@gmail.com}.
 Partially supported, for this work, from FET programme within the EU FP7, under the SIMBAD project, contract 213250. Currently, (01/10/2012-31/09/2013), Marie Curie fellow under the ERCIM “Alain Bensoussan” Fellowship Programme.}
  \and Albert Sol{\'e}-Ribalta\thanks{Universitat Rovira i Virgili - Spain  \email{\{albert.sole,francesc.serratosa\}@urv.cat}.
Research supported in part by Consolider Ingenio 2010 project (CSD2007-00018) and from the CICYT project (DPI2010-17112).}
   \and Marcello Pelillo \thanks{Universit{\`a} Ca' Foscari Venezia - Italy.
\email{pelillo@dsi.unive.it}.
}
    \and Francesc Serratosa \footnotemark[3]   
   }

\begin{document}

\maketitle

\vspace{-0.25in}

\begin{abstract}
In structural pattern recognition, given a set of graphs, the computation of a Generalized Median Graph is a well known problem. Some methods approach the problem by assuming a relation between the Generalized Median Graph and the Common Labelling problem. However, this relation has still not been formally proved. In this paper, we analyse such relation between both problems. The main result proves that the cost of the common labelling upper-bounds the cost of the median with respect to the given set. In addition, we show that the two problems are equivalent in some cases.
\end{abstract}

\section{Introduction}

In many pattern recognition applications, we are given a set of different representations of the same object and the goal is to summarize these representations into a single one. The resulting representation should capture the important features of the object and discard noisy or unexpected variations. When the representation is made using attributed graphs, this graph is identified as the Generalized Median Graph \cite{Jiang2001}, or simply the Median Graph. Given a training set of graphs, the Median Graph is formally defined as a graph which minimizes the sum of costs to all other graphs in the set.

If we assume that vertices are not uniquely labelled, like in \cite{Dickinson2004}, the problem of finding the Median Graph is, in its general form, at least as difficult as the problem of matching two graphs under a particular cost function, e.g. the Graph Edit Distance, which is a NP-Hard problem \cite{astar09}. Indeed, the Median Graph cannot be computed in closed form since its synthesis depends on the matchings between itself and the given graphs and the matchings to the Median Graph clearly require having the Median Graph. A usual way to deal with this chicken-egg problem is using an incremental approach where the Median Graph is coarsely constructed and then iteratively refined until all graphs in the training set are considered. Several approaches address the problem in this fashion \cite{Ferrer2005,Ferrer2007,Ferrer2010,Hlaoui2006,Jain2004,Jain2010}. A completely different approach to compute the Median Graph is to decouple the matchings and the synthesis process. This approach relies on the assumption that given the vertex labellings that compute the Median Graph, its computation can be, in most applications, done efficiently in polynomial time, e.g. averaging the vertices and edge attributes. This approach can be summarized in two steps. In the first step, we obtain a Common Labelling between the given graphs. The objective of the Common Labelling, initially defined in \cite{Sole-Ribalta2010,Sole-Ribalta2011}, is to minimize the pair-wise labellings among a set of graphs with some transitivity restrictions. Once we know this information, we can easily compute an Approximated Median Graph. Figure \ref{figureIntro} illustrates the complete process to generate a Median Graph using a Common Labelling. Note the given set of graphs is labelled to a virtual node set and Median Graph is not computed until the end of the process. The main advantage of using a Common Labelling approach for approximating the Median Graph relies on the fact that the Median Graph does not need to be computed until the end of the process. In this way, labellings of the initial graphs to the Median Graph are not needed and the initial chicken-egg problem disappears.

Several works exist in the literature which decouple the problem of the Median Graph computation. The first method to completely decouple the matching process from the synthesis process was presented by Hlaoui and Wang \cite{Hlaoui2006}. Another recent method, based on linear programming, has been proposed in \cite{Justice2006} and \cite{Mukherjee2009}. But possibly the most complete work on these kind of methods is presented in \cite{Sole-Ribalta2012}. Experiments in \cite{Sole-Ribalta2012} show that using the Common Labelling for computing the Median Graph gives satisfactory results, but up to now a formal relation between the Common Labelling problem and the Median Graph synthesis was missing. In this work, we show that, if the cost for matching graphs is a metric, the two problems are tightly connected because we can bound the Median Graph error using the Common Labelling value. The obtained bounds show that, when the error of the Common Labelling is low, the obtained graph median is close to the real one. In addition, in the specific case of unattributed graphs with the squared Euclidean distance as cost function, the two problems are equivalent.

\begin{figure}[h!]
\centering
\includegraphics[width=1\textwidth]{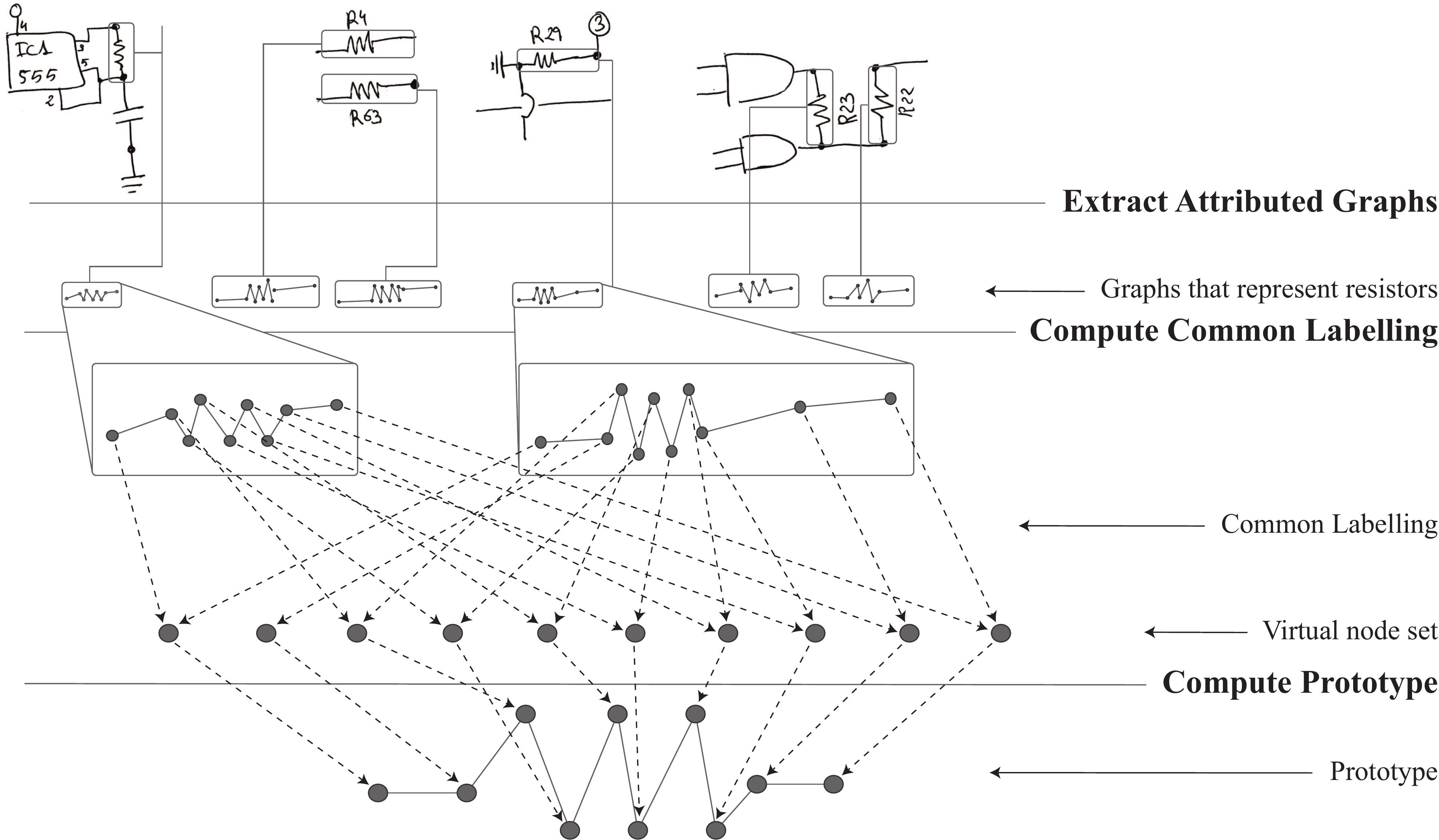}
\label{figureIntro}
\caption{The process for computing an Approximated Median Graph with the Common Labelling. After representing the given objects, which in this figure are sketches of electrical circuits, with attributed graphs we look for labelling the nodes of each graph. The virtual set of nodes does not have structure and is used to compare the labelling for each graph and evaluate their pairwise matching cost. After choosing a labelling for each graph we convert the virtual set of nodes into the actual graph prototype. See section \ref{prob} for a formal definition of the Common Labelling problem.}
\end{figure}

\section{Definitions}
Let $\mathcal{H}$ be a set of attributed graphs representing the input/output space of our problems. Each graph is represented as a tuple $G = (V, E, A_{V}, A_{E})$, where $V=\{v_{1}, ... ,v_{n}\}$ represents the vertex set, $E \subseteq \{e_{a,b}, \forall a,b \in {1..n}\}$ the edge set, and functions $A_{V}: V \rightarrow D_{V}$ and $A_{E}: E \rightarrow D_{E}$ assign attributes to vertex and edges respectively.

Given a set of $m$ attributed graphs $S=\{G_{1},...,G_{m}\}, G_{i}=(V_{i},E_{i},A_V,A_E) \in \mathcal{H}$, we assume that each of these graphs have the same number of vertices $n$. If this is not the case, several solutions have been proposed to extend the size of the graphs \cite{Jain2010,Wong1985}. However, the most common approach is to include null vertices \cite{Wong1985} which represent deletions and insertion of vertices in the resulting labelling. In the general graph matching setting, vertices of each graph are not uniquely identified by their index, i.e. we cannot assign or identify $v_{3} \in V_{1}$ with $v_{3}\in V_{2}$ only because they have the same index $3$. Indeed, the difficult part of comparing a pair of graphs relies in finding a suitable bijection $\pi$ of vertices which provides the right ordering. In the following, given a bijection $\pi$, the notation $G^{\pi}$ means that $v_{i}(\pi) = v_{\pi(i)}$, so that vertices of $V$ and edges $E$ are permuted accordingly to $\pi$. The bijection $id \in \Pi$ represents the identity $v_i(id)=v_i$. Figure \ref{figureTHM2} (a) shows how graphs are permuted to a common reference system with permutations $\pi_i$ and $\rho_i$.
The function $c: \mathcal{H} \times \Pi \times \mathcal{H} \times \Pi \rightarrow R^{+}$ is a user-defined cost between two graphs whose vertices have a fixed bijection. We assume that $c(\cdot,\cdot,\cdot,\cdot)$ can be computed efficiently in polynomial time because the vertex to vertex correspondence is fixed, and consequently also the edge to edge correspondence and their attributes. We use the shorthand $c(G_{i}^{\pi_{i}},G_{j}^{\pi_{j}})=c(G_{i},\pi_{i},G_{j},\pi_{j})$ and $c(G_{i},G_{j})=c(G_{i},id,G_{j},id)$. If the cost function is a metric we denote it as $c_{M}$ in this case, given a fixed set of bijections $\pi_{1,\dots,m}$, the following axioms hold:

\begin{description}
\item[identity] $\metriccost{G_{i}}{}{G_{j}}{} = 0 \Leftrightarrow G_{i}=G_{j}$,
\item[positivity] $\metriccost{G_{i}}{}{G_{j}}{} \geq 0$,
\item[symmetry] $\metriccost{G_{i}}{}{G_{j}}{} = \metriccost{G_{j}}{}{G_{i}}{}$,
\item[triangle inequality] $\metriccost{G_{i}}{}{G_{j}}{} \leq \metriccost{G_{i}}{}{G_{k}}{}+\metriccost{G_{k}}{}{G_{j}}{}$.
\end{description}

We define the distance $d$ between two graph as the minimum cost among all possible bijections of attributes in vertices and edges. That is,

\begin{equation}
d(G_{1},G_{2}):=\min_{\begin{array}{c}
\pi_{1},\pi_{2} \in \Pi
\end{array}}{\metriccost{G_{1}}{{\pi_{1}}}{G_{2}}{{\pi_{2}}}}
\end{equation}

Given a set of graphs $S=(G_{1},...,G_{m}) \subseteq \mathcal{H}$, the Generalized Median Graph \cite{Jiang2001} is defined as a graph $G^{*}$, taken from the set $\mathcal{H}$, which minimizes the average sum of costs to all graphs in $S$:

\begin{equation}
\GM{(\mathcal{H})} := \min_{ \begin{array}{c}
\rho_{1},\dots,\rho_{m} \in \Pi \\
G \in \mathcal{H}
\end{array}}{\dfrac{1}{m}\sum_{i=1}^{m}{\cost{G_{i}}{\rho_{i}}{G}{}}}
\end{equation}

If not explicitly stated the argument of $\GM{}$ is $\mathcal{H}$. In the following, and as Figure \ref{figureTHM2} (a) shows, we will denote with  $\rho_i$ the permutations which obtain the Median Graph.

%%%%%%%%%%%%%%%%%%%%%%%%%%%%%%%%%%%%%%%%%%%%%%%%%%%%%%%%%%
\section{The Common Labelling Problem}
\label{prob}
%%%%%%%%%%%%%%%%%%%%%%%%%%%%%%%%%%%%%%%%%%%%%%%%%%%%%%%%%%

Given a set of graphs $S=(G_{1},...,G_{m}) \subseteq \mathcal{H}$, the Common Labelling problem aims at finding a, possibly low cost, consistent multiple isomorphism between the graphs, such that for every three mappings $\pi_{i,j},\pi_{j,r}$ and $\pi_{i,r}$ we have $\pi_{i,j} \circ \pi_{j,r} = \pi_{i,r}$. Equivalently, we look for $m$ consistent bijections assigning vertices of the graph of a virtual vertex set and that minimize the average sum of pairwise distances between graphs in $S$. Its normalized objective function is the following:

\begin{equation}
\CL{} := \min_{\begin{array}{c}
\pi_{1},\dots,\pi_{m} \in \Pi
\end{array}}{\dfrac{1}{m^{2}} \sum_{i=1}^{m}{\sum_{j=1}^{m}{\cost{G_{i}}{\pi_{i}}{G_{j}}{\pi_{j}}}}}
\end{equation}

Once the Common Labelling and the $m$ bijections $\pi_{1,\dots,m}$ that computes the value are obtained, we assume that we can efficiently estimate a median graph $\overline{G}$:

\begin{equation}
\overline{G} \in \argmin_{\begin{array}{c}
G \in \mathcal{H}
\end{array}}{\sum_{i=1}^{m}{\cost{G_{i}}{\pi_{i}}{G}{}}}
\label{medianGraph}
\end{equation}

\noindent which we call Approximated Median Graph. In the following, and as Figure \ref{figureTHM2} (a) shows, we will denote with $\pi_i$ the permutations which obtain the Approximated Median Graph through the Common Labelling.

%%%%%%%%%%%%%%%%%%%%%%%%%%%%%%%%%%%%%%%%%%%%%%%%%%%%%%%%%%
\section{Relating the Common Labelling with the Generalized Median Graph}
\label{relatingCommonLabelling}
%%%%%%%%%%%%%%%%%%%%%%%%%%%%%%%%%%%%%%%%%%%%%%%%%%%%%%%%%%

In this section, we show two main results of this work. The first theorem shows the relationship between the objective function of the Common Labelling, $\CL{}$, and the objective function of the Median Graph, $\GM{}$. The second theorem shows that, if the functional of the Common Labelling $\CL{}$ has a low value, the Approximated Median Graph $\overline{G}$ is close to the Median Graph $G^{*}$.

\begin{theorem}
Let $\mathcal{H}$ be a set of graphs and $S=\{G_{1},\dots,G_{m}\}$ a subset of $\mathcal{H}$. In addition, let $\overline{G}$ be the Approximated Median Graph computed considering $S$ and the bijections obtained by the Common Labelling, $\pi_{1,\dots,m}$. Let the cost function $c_M$ be a metric. Then

\begin{equation}
\CL{} \geq \GM(\{\overline{G}\}) \geq \GM{} \geq \dfrac{1}{2}\CL{}
\label{bounds}
\end{equation}

\label{thm:bounds}
\end{theorem}

\begin{proof}
We start with the left hand side of (\ref{bounds}):

\begin{equation}
\begin{array}{rcl}
\CL{} & = & \dfrac{1}{m^{2}} \displaystyle \sum_{i=1}^{m}{\sum_{j=1}^{m}{\metriccost{G_{i}}{\pi_{i}}{G_{j}}{\pi_{j}}}}\\
& \geq & \dfrac{1}{m^{2}} \displaystyle \sum_{i=1}^{m}{\sum_{j=1}^{m}{\metriccost{\overline{G}}{}{G_{j}}{\pi_{j}}}}\\
& \geq & \dfrac{1}{m} \displaystyle \sum_{j=1}^{m}{\metriccost{\overline{G}}{}{G_{j}}{\pi_{j}}}\\
& = & \GM(\{\overline{G}\}) \\
%& \geq & \dfrac{1}{m} \displaystyle \sum_{j=1}^{m}{\metriccost{G_{j}}{\rho_{j}}{{G^{*}}}{}}\\
& \geq & \GM{} \\
\end{array}
\end{equation}
The second step comes from optimality of the Approximated Median Graph, see (\ref{medianGraph}).

\noindent The right hand side of (\ref{bounds}) follows from:
\begin{equation}
\begin{array}{rcl}
\GM{} & = & \dfrac{1}{m} \displaystyle\sum_{i=1}^{m}{ \metriccost{G_{i}}{\rho_{i}}{{G^{*}}}{}} \\
& = & \dfrac{1}{2m^{2}} \displaystyle\sum_{i=1}^{m}{\sum_{j=1}^{m}{\metriccost{G_{i}}{\rho_{i}}{{G^{*}}}{}+\metriccost{{G^{*}}}{}{G_{j}}{\rho_{j}}}}\\
& \geq & \dfrac{1}{2m^{2}} \displaystyle\sum_{i=1}^{m}{\sum_{j=1}^{m}{\metriccost{G_{i}}{\rho_{i}}{G_{j}}{\rho_{j}}}}\\
& \geq & \dfrac{1}{2m^{2}} \displaystyle\sum_{i=1}^{m}{\sum_{j=1}^{m}{\metriccost{G_{i}}{\pi_{i}}{G_{j}}{\pi_{j}}}} \\
& = & \dfrac{1}{2}\CL{} \\
\end{array}
\end{equation}

The third step uses the triangle inequality and the forth step comes from considering the optimality of $\pi_i$ and $\pi_j$.
\end{proof}

\begin{figure}[htp]
\centering
\includegraphics[width=1\textwidth]{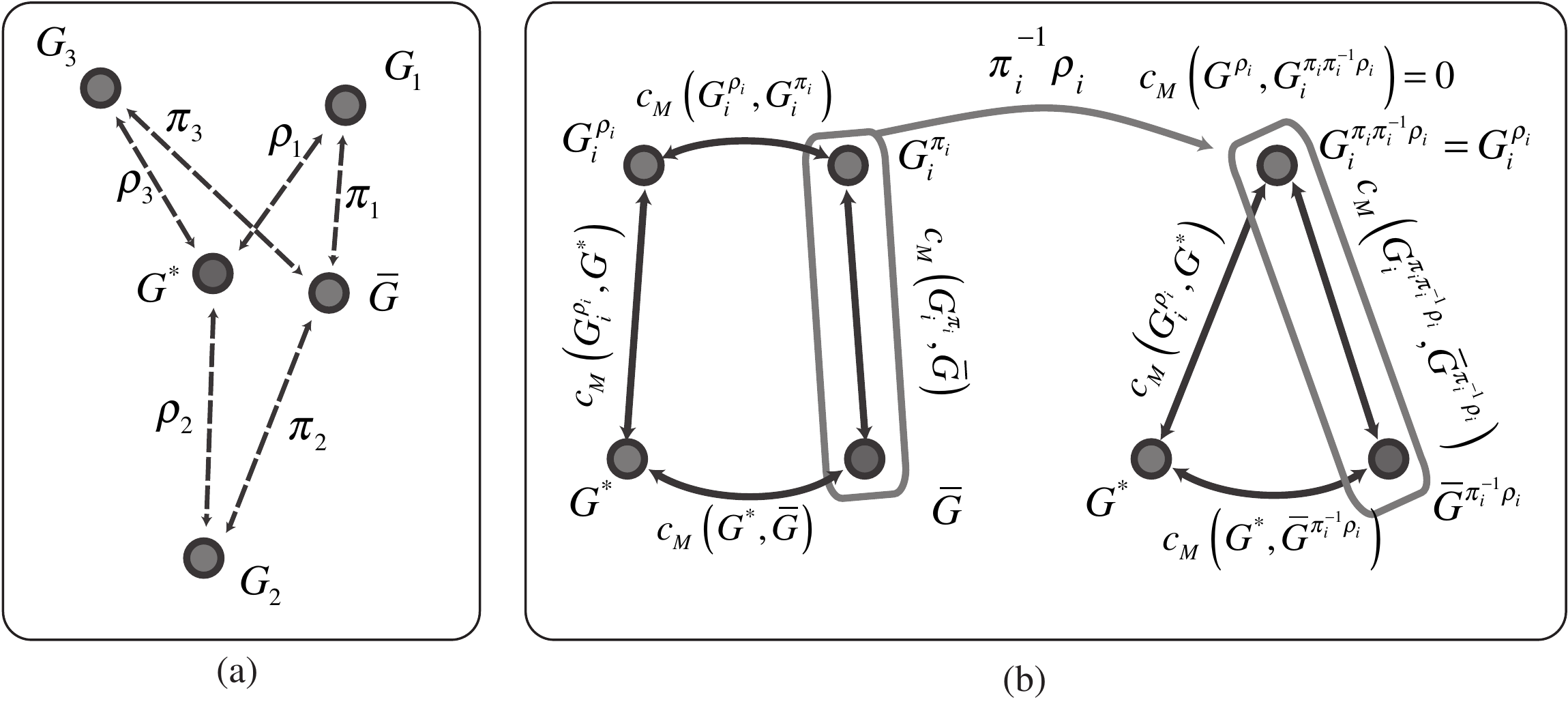}
\caption{(a) Notation for Theorems \ref{thm:bounds} and \ref{thm:distancemedian}. (b) Graphical representation of (\ref{thm2inequality}), which is the basic inequality for proving Theorem \ref{thm:distancemedian}.}
\label{figureTHM2}
\end{figure}

\begin{theorem}
Let $\mathcal{H}$ be a set of graphs and $S=\{G_{1},\dots,G_{m}\}$ be a subset of $\mathcal{H}$. In addition, let $\overline{G}$ be the Approximated Median Graph computed considering $S$ and $G^{*}$ the Generalized Median Graph. Then,

\begin{equation}
d(\overline{G},G^{*}) \leq 2\CL{} \leq 4\GM{} 
\label{distancemedian}
\end{equation}

\label{thm:distancemedian}
\end{theorem}

\begin{proof}
Let $\pi_{1,\dots,m}$ be the bijections obtained by the Common Labelling and $\rho_{1,\dots,m}$ the bijections related to $G^{*}$ and $c_M$ a metric cost function. Since $c_{M}$ is a metric we have for each single graph $G_i$: 

\begin{equation}
\begin{array}{lcl}
c_M({G^{*}},{\overline{G}}) &\leq & c_M({G^{*}},{G_i}^{\rho_i}) + c_M({\overline{G}},{G_i}^{\rho_i}) \\
&\leq & c_M({G^{*}},{G_i}^{\rho_i}) +c_M({G_i}^{\rho_i},{G_i}^{\pi_i}) + c_M({\overline{G}},{G_i}^{\pi_i}) 
\end{array}
\label{thm2inequality}
\end{equation}
since $\rho_i$ and $\pi_i$ may be different $c_M({G_i}^{\rho_i},{G_i}^{\pi_i})\neq 0$. However, applying bijection $\pi^{-1}_i\rho_i$ to ${G_i}^{\pi_i}$ and $\overline{G}$ costs are preserved $c_M(\overline{G},{G_i}^{\pi_i})=c_M(\overline{G}^{\pi^{-1}_i\rho_i},{G_i}^{\rho_i})$ and $c_M({G_i}^{\rho_i},{G_i}^{\pi_i\pi^{-1}_i\rho_i})=0$.
This reasoning is visualized in Figure \ref{figureTHM2} (b).  Hence, 

\begin{equation}
\begin{array}{lcl}
\metriccost{\overline{G}}{\pi^{-1}_{i}\rho_{i}}{{G^{*}}}{} & \leq & \metriccost{\overline{G}}{\pi^{-1}_i\rho_i}{G_{i}}{\rho_{i}} + \metriccost{G_{i}}{\rho_{i}}{{G^{*}}}{}
\end{array}.
\label{bijectionToCL}
\end{equation}
In (\ref{bijectionToCL}), vertices and edges of $G_{1},\dots,G_{m}$ and $\overline{G}$ have been permuted accordingly to $G^{*}$. To ease notation, assume that $\pi_i$ correspond to the identity. Consequently,

\begin{equation}
d(\overline{G},G^{*}) \leq \metriccost{\overline{G}}{\rho_{i}}{{G^{*}}}{} \leq \metriccost{\overline{G}}{\rho_{i}}{G_{i}}{\rho_{i}}+\metriccost{G_{i}}{\rho_{i}}{{G^{*}}}{}.
\label{triangleineq}
\end{equation}
Then, adding inequality (\ref{triangleineq}) for the different $G_i$\rq{}s we get:

\begin{equation}
\begin{array}{rcl}
d(\overline{G},G^{*})  &\leq & \dfrac{1}{m} \sum_{i=1}^{m}{\metriccost{\overline{G}}{\rho_{i}}{G_{i}}{\rho_{i}}+\metriccost{G_{i}}{\rho_{i}}{{G^{*}}}{} }\\
& = & \GM(\{\overline{G}\})+\GM{}\\
& \leq & 2\CL{}\\
\end{array}
\label{thm2ineq}
\end{equation}

\end{proof}

A desirable output for the user is that the Approximated Median Graph is an $\epsilon$ approximation of the given objects. The following corollary shows that, in this case, this Approximated Median Graph is close to the actual Median Graph.

\begin{corollary}
Let $S=\{G_{1},\dots,G_{m}\}$ admit an Approximated Median Graph $\overline{G}$ such that $\GM(\{\overline{G}\}) \leq \epsilon$. Then $ d(\overline{G},G^{*}) \leq 3 \epsilon$.
\label{corollary1}
\end{corollary}

The proof is based on equation (\ref{thm2ineq}) of theorem \ref{thm:distancemedian} and is left to the reader.

\noindent Theorem \ref{thm:bounds} and \ref{thm:distancemedian} are proven considering the optimal computation of $\CL{}$. If we relax this assumption with a suboptimal computation we get the following corollary.

\begin{corollary}
Let $\mathcal{H}$ be a set of graphs and $S=\{G_{1},\dots,G_{m}\}$ be a subset of $\mathcal{H}$. In addition, let $\overline{G^{'}}$ be the Approximated Median Graph computed considering $S$ and the, possibly suboptimal, bijections obtained by Common Labelling whose value is $\textnormal{CL}$. Let the cost function $c_M$ be a metric. Then $\textnormal{CL} \geq \textnormal{GM}(\overline{G^{'}}) \geq \GM{}$ and $d(\overline{G^{'}},G^{*}) \leq 2\textnormal{CL}$.
%\begin{itemize}
%\item $\textnormal{CL} \geq \GM{}$
%\item $d(\overline{G^{'}},G^{*}) \leq 2\textnormal{CL}$
%\end{itemize}
\end{corollary} 

%%%%%%%%%%%%%%%%%%%%%%%%%%%%%%%%%%%%%%%%%%%%%%%%%%%%%%%%%%
\section{Median Graph of Weighted Graphs}
%%%%%%%%%%%%%%%%%%%%%%%%%%%%%%%%%%%%%%%%%%%%%%%%%%%%%%%%%%

Clearly, the notion of Median Graph can be used with a large set of different cost functions. In this section, we will show how using the original proposed cost \cite{Jiang2001} between graphs and restricting to weighted graphs, the Median Graph problem reduces exactly to the Common Labelling problem. Let $A_{V}: V \rightarrow [0,1]$ and $A_{E}: E \rightarrow [0,1]$ be the domain of vertices and edges attributes. In this case, the value \lq\lq{}1\rq\rq{} indicates that the graph vertex, or edge, exists and value \lq\lq{}0\rq\rq{} that the vertex, or edge, does not exist. We use a vector/matrix representation, so that $\mathbf{V_i}(r)=A_{V}(v_r)$ and $\mathbf{E_i}(r,s)=A_{E}(e_{r,s})$ where $v_r \in V_i$ and $e_{r,s} \in E_i$ and bijections $\pi_{i}$ are represented as permutation matrices $\mathbf{p_{i}}$. In case no vertex position is indicated, $\mathbf{V_i}$, we refer to the complete vector.
% %Hence, each graph is represented as follows:
%
%\begin{equation}
%\begin{array}{c c}
%\mathbf{V}_p[a] = \left\{
%\begin{array}{l l}
%1 & \quad \text{if $v_a \in V_p$}\\
%0 & \quad \text{if $v_a \notin V_p$}\\
%\end{array} \right.
%&
%, \mathbf{E}_p[a,b] = \left\{
%\begin{array}{l l}
%1 & \quad \text{if $e_{a,b} \in E_p$}\\
%0 & \quad \text{if $v_{a,b} \notin E_p$}\\
%\end{array} \right.
%\end{array}
%\label{unattributedGraphCodification}
%\end{equation}

As cost function we use the squared Euclidean distance, $c(v_{r},v_{s})=\|\mathbf{V_i}(r)-\mathbf{V_j}(s)\|^{2}$ where $v_{r} \in V_i$ and $v_{s} \in V_j$. The edge cost function is defined in an equivalent form. This cost was also used in the genetic algorithm of \cite{Jiang2001} where authors proved the best prototype for a set of graphs, with fixed labellings, is the average of attributes

\begin{equation}
\left\lbrace 
\begin{array}{lcl}
\overline{\mathbf{V}}(r) &=& \displaystyle \dfrac{1}{m}\sum_{k=1}^{m}{\mathbf{V_{k}}(r)}\\
\overline{\mathbf{E}}(r,s) &=& \displaystyle \dfrac{1}{m}\sum_{k=1}^{m}{\mathbf{E_{k}}(r,s)}
\end{array}
\right.
\label{mean}
\end{equation}
Under these considerations, we can state the following theorem: 

\begin{theorem}
Let $\mathcal{H}$ be a set of weighted graphs, $S=\{G_{1},\dots,G_{m}\}$ a given subset of $\mathcal{H}$. and $p_{1,\dots,m} \in \mathds{R}^{N \times N}$ $m$ permutation matrices. Considering the cost given by the squared Euclidean distance, we have:

\begin{equation}
\frac{1}{2} \CL{} = \GM{}.
\end{equation}

%\begin{eqnarray}
%\frac{1}{2} CL^* &=& \frac{1}{2m^2} \sum_{p=1}^{m}\sum_{p=1}^{m} || \mathbf{V}_p^T \pi_p - \mathbf{V}_q^T \pi_q ||_2^2 + || \pi_p^T \mathbf{E}_p\pi_p - \pi_q^T\mathbf{E}_q\pi_q||_F^2 \\
%&=& \frac{1}{m} \sum_{p=1}^{m} || \mathbf{V}_p^T \pi_p - \mathbf{V}_{M}^T \pi_{M} ||_2^2 + || \pi_p^T \mathbf{E}_p\pi_p - \pi_M^T\mathbf{E}_M\pi_M||_F^2 \\
%&=& GM^* 
%\end{eqnarray}
\label{thm3}
\end{theorem}

\begin{proof}
The scalar product of two vectors is:

\begin{equation}
\scalprod{\mathbf{V_{i}}}{\mathbf{V_{j}}}=\displaystyle \sum_{r=1}^{n}{\mathbf{V_{i}}(r)\mathbf{V_{j}}(r)}.
\nonumber
\end{equation}

The proof follows the lines of the Huygens theorem \cite{huygens}.

\begin{equation}
\begin{array}{rcl}
\frac{1}{2} \CL{} &=& \displaystyle \frac{1}{2m^2} \sum_{i=1}^{m}\sum_{j=1}^{m} \sen{\argp{p}{i}} -2 \scalprod{\argp{p}{i}}{\argp{p}{j}} + \sen{\argp{p}{j}}\\

 &=& \displaystyle \frac{1}{m^2} \sum_{i=1}^{m}\sum_{j=1}^{m} \sen{\argp{p}{i}} - \scalprod{\argp{p}{i}}{\argp{p}{j}} \\

 &=& \displaystyle \frac{1}{m} \sum_{i=1}^{m}{\sen{\argp{p}{i}} + \sum_{i=1}^{m}\sum_{j=1}^{m} - \dfrac{2}{m^2}\scalprod{\argp{p}{i}}{{\argp{p}{j}}}} + \dfrac{1}{m^2}\scalprod{{\argp{p}{i}}}{{\argp{p}{j}}}\\
 
 &=& \displaystyle \frac{1}{m} \sum_{i=1}^{m}{(\sen{\argp{p}{i}} - \dfrac{2}{m}\scalprod{\argp{p}{i}}{\sum_{j=1}^{m}{\argp{p}{j}}})} + \scalprod{\dfrac{1}{m}\sum_{i=1}^{m}{\argp{p}{i}}}{\dfrac{1}{m}\sum_{j=1}^{m}{\argp{p}{j}}} \\

 &=& \displaystyle \frac{1}{m} \sum_{i=1}^{m}{\sen{\argp{p}{i}}-2\scalprod{\argp{p}{i}}{\overline{\mathbf{V}}}+\sen{\overline{\mathbf{V}}}} \\

&\geq& \GM{}\\

\end{array}
\end{equation}

The converse inequality is similarly proved and the process is equivalent for the edge costs.
\end{proof}

As an immediate consequence of theorem \ref{thm3} we have that the Approximated Median Graph error is the same as the Generalized Median Graph. The proof is based on theorem \ref{thm3} and is left to the reader. 

\begin{corollary} Under the hypothesis of theorem \ref{thm3} we have:
\begin{equation}
\GM{(\{\overline{G}\})}=\GM{}
\end{equation}
\end{corollary}

By exploiting the particular properties of the squared Euclidean distance, which is not a metric, we get a much stronger result than theorem \ref{thm:bounds}.

%%%%%%%%%%%%%%%%%%%%%%%%%%%%%%%%%%%%%%%%%%%%%%%%%%%%%%%%%%
\section{Discussion}
%%%%%%%%%%%%%%%%%%%%%%%%%%%%%%%%%%%%%%%%%%%%%%%%%%%%%%%%%%

In this paper we analysed the relation between two structural pattern recognition problems, the Median Graph and the Common Labelling. We proved that these problems are closely related and in some special cases they are in fact equivalent, thereby formalising a connection which up to now was unknown. This connection confirms that algorithms based on the Common Labelling, to compute the Median Graph, are theoretically sound. In addition, the proposed bounds are useful in practice, when the Common Labelling is computed using non-exact algorithms, like in \cite{Sole-Ribalta2011}.

\bibliographystyle{splncs}
\bibliography{literature}

\end{document}